\documentclass[10pt]{colt2018} % Anonymized submission
\usepackage{url}
\usepackage{nicefrac} 
\usepackage{graphicx}
\usepackage{algorithm}
\usepackage{algorithmic}
\usepackage{amsmath}
\usepackage{subcaption}
\usepackage{natbib}
\usepackage{pgfplots}
\usepackage{hyperref}
\usepackage{enumitem}
\usepackage{microtype}      % microtypography

\input{Definitions}
\title[Learnability of TVGAM ]{Statistical Learnability of Generalized Additive Models \\based on Total Variation Regularization}
\usepackage{times}
 \coltauthor{\Name{Shin Matsushima}  \Email{shin\_matsushima@mist.i.u-tokyo.ac.jp}\\
 \addr 7-3-1, Hongo, Bunkyo, Tokyo, Japan
 }

\begin{document}

\maketitle
\newcommand{\nnbr}[1]{{\rm TV}\!\rbr{#1}}
\newcommand{\ex}{\vec{x}}
\newcommand{\range}[2]{[#1,#2]}

\begin{abstract}
A generalized additive model (GAM,~\cite{HasTib87}) is a nonparametric model by the sum of univariate functions with respect to
each explanatory variable, i.e., $f(\ex) = \sum f_j(x_j)$, where $x_j\in\RR$ is $j$-th component of a sample $\ex\in \RR^p$.
In this paper, we introduce the total variation (TV) of a function as a measure of the complexity of
functions in $L^1_{\rm c}(\RR)$-space.
Our analysis shows that a GAM based on TV-regularization 
exhibits a Rademacher complexity of $O(\sqrt{\nicefrac{\log p}{m}})$, which is tight in terms of both $m$ and $p$ in the agnostic case of the classification problem.
In result, we obtain generalization error bounds for finite samples according to work by \citet{Bartlett02}.
\end{abstract}

\section{Introduction}
In this paper, 
we focus on the learning problem of the following form of prediction functions:
\begin{align}\label{eq:GAM}
f(\ex) = \sum_{j\in\range1p} f_j(x_j),
\end{align}
where $\ex\in\RR^p$ denotes a sample and $x_j\in\RR$ denotes the
$j$-th explanatory variable for each $j\in\range1p\triangleq\cbr{j\in \NN|1\le j\le p}$.
This was first proposed by \cite{HasTib87} and is known as a \emph{generalized additive model} (GAM).
In this paper, we call $f_j(\cdot)$ a \emph{weight function} and $f(\cdot)$ a \emph{GAM predictor}.
This not only includes linear predictors but also captures nonlinear 
relationships between explanatory variables and the targeted values. Although 
complex interactions or dependencies among explanatory values are not expressed,
GAM predictors are expected to exhibit higher predictive performance
when properly learned from a sufficiently large amount of data, at least in comparison with simple linear models. There has already been substantial work on data mining and statistics using GAMs~\citep{GuiEdwHas02,Wood06}. Learning GAM is mostly conducted by regularized empirical risk minimization, in which its regularization is based on the wiggliness of weight functions and is reduced to fitting natural cubic splines~\citep{Wood06,Friedman01}. 
%Probably, a lack of statistical foundationhas prohibited further development of this methodology in spite of its attractive traits such as representabitiy and interepretability.

We consider a totally different situation in this paper, where the \emph{total variation} (TV) of a function is employed as a measure of complexity of functions in $L^1_{\rm c}(\RR)$-space. Here, $L^1_{\rm c}(\RR)$ denotes a space of functions with compact support in $L^1$-space on $\RR$.
We first introduce the definition of TV and a class of GAM predictors
regularized by the sum of TV among all weight functions:
\begin{definition}
 For $f\in L^1_{\rm c}(\RR)$, total variation of $f$ denoted by $\nnbr{f}$ is defined as follows
\begin{align*}
\nnbr{f(\cdot)} \triangleq \sup_{\substack{x: \NN\to\RR, \, {\rm increasing}}} \sum_{n\in\NN} \abr{f(x(n)) - f(x(n+1))}.
\end{align*}
% \begin{align}
%   \nnbr{f}  = \sup \cbr { \int_\RR f(x) g'(x) dx : g(x)\in C^1_{\rm c}, \nbr{g}_{\infty}\le 1 }. 
% \end{align}
% Here, $C^1_{\rm c}$ denotes a set of continuously differentiable functions with compact support.
 %$\nbr{\cdot}_{\infty}$ denotes the norm in $L^{\infty}$-space. 
\end{definition}
\begin{definition}
  For $C\in\RR_+$ and $p\in\NN$, a set of TV-regularized GAM predictors denoted by ${\rm GAM}_p(C)$ is defined as follows: 
  \begin{align}
  {\rm GAM}_p(C) \triangleq \cbr{f \in L_{\rm c}^1(\RR^p) \middle| f(\ex) = \sum_{j\in\range1p} f_j(x_{j}),  \sum_{j\in\range1p} \nnbr{f_j} \le C}.
  \end{align}
\end{definition}
As we discuss in the next section, it has several desirable properties as a measure of complexity of functions under the framework of regularized empirical risk minimization.
The main theorem of this paper (Theorem~\ref{thm:main}) states that the empirical Rademacher complexity of ${\rm GAM}_p(C)$ has an order of $O(C\sqrt{\nicefrac{\log p}{m}})$. In result, we obtain generalization error bounds for finite samples according to work by \citet{Bartlett02}.

The main theorem is shown by analysis of the empirical Gaussian complexity without using concentration inequalities 
but with basic inequalities known in the field of stochastic process.
This result implies that even discontinuous functions are learnable in GAM based on TV-regularization.
For the paper to be self-contained, we state the definition of the empirical Rademacher complexity and the empirical Gaussian complexity below:
\begin{definition}
  For a set of functions $F\subset L^1(X)$ and $(\ex_i)_{i\in \range1m} \in X^m$,
the empirical Gaussian complexity of $F$ (with respect to $\ex^m$) denoted by $G(F,(\ex_i)_{i\in \range1m})$ is defined as follows:
 \begin{align}
  G(F,(\ex_i)_{i\in \range1m}) \triangleq \frac 1m \EE_{\gammab} \sup_{f\in F} \abr{\sum_{i\in\range1m} \gamma_i f(\ex_i)},
 \end{align}
 where each component of $\gammab = (\gamma_i)_{i\in\range1m}$ is an independent standard Gaussian random variable.
Similarly, the empirical Rademacher complexity denoted by $R(F,(\ex_i)_{i\in \range1m})$ is defined as follows
 \begin{align}
   R(F,(\ex_i)_{i\in \range1m}) \triangleq \frac 1m \EE_{\epsilonb} \sup_{f\in F} \abr{\sum_{i\in\range1m} \varepsilon_i f(\ex_i)}.
 \end{align}
 where each component of $\epsilonb=(\varepsilon_i)_{i\in\range1m}$ is an independent Rademacher random
 variable\footnote{A Rademacher random variable refers to a Bernoulli
   random variable of $\cbr{\pm 1}$ with the parameter $\half$.}.
\end{definition}

In Section 2, we introduce several properties of the TV-regularization. In Section 3, we show a technical lemma that will be used to prove the main theorem, which is formally stated and proven in Section 4. We conclude
our paper with some discussion on related work and the tightness of our bound in Section 5.
\section{Property of Total Variation (TV)}
The formal expression of the TV-regularized empirical risk minimization for a GAM is given as follows: 
\begin{align} \notag
&{\rm minimize}_{ f, f_j } &&\sum_{i\in\range1m} \ell \rbr{f(\ex_{i}), y_i } + \lambda \sum_{j\in\range1p} \nnbr{f_j(\cdot)} \nonumber \\
&{\rm subject\ to } && 
f(\ex) = \sum_{j\in\range1p} f_j(x_j),\, f_j(\cdot) \in L^1_c(\RR).   
\label{eq:FLR0}
\end{align}
First, we state a type of compatibility of 
TV of weight functions to $L^1$-norm of a weight vector of linear predictors.
\begin{property}
If $\omega$ is differentiable, then it holds that
\begin{align}
\nnbr{\omega} = \nbr{\omega'}_1 = \int_{\RR} \abr{\omega'(x)}{\rm d}x.
\end{align}
\end{property}
When we further restrict weight functions 
to be in the form of $f_j(x_j)=w_jx_j[\![-M\le x_j \le M]\!]$\footnote{$[\![ \bullet ]\!]$ is a function that returns 1 if $\bullet$ is true and 0 otherwise.} where $M>0$, the TV of $f_j$ coincides with the $L^1$-norm of $(w_j)_{j\in\range1p}$ and hence
the problem \eqref{eq:FLR0} is reduced to problems such as $L1$-logistic regression and LASSO~\citep{Tib96,Tib97}. 

Second, we state that TV has a type of invariance and
this property leads to a desirable property of a solution of the problem~\eqref{eq:FLR0}.
\begin{property}
Given a strictly monotonic function $\varphi:\RR\to\RR$, 
it holds that 
\begin{align}
\nnbr{\omega(\cdot)} = \nnbr{\omega\circ\varphi(\cdot)}.
\end{align}
\end{property}
\begin{proof}
$\nnbr{\omega}$ can be seen as the total variation of the signed measure $\mu(A)=\int_A\omega(x){\rm d}x$ and the total variation of measures is invariant under bijective continuous transformations~\citep{Gyo02}.
\end{proof}
We note that this can be easily confirmed in the case where both $\omega$ and $\varphi$ are differentiable as follows:
\begin{align*}
\int_{\RR} \abr{\frac{{\rm d}\omega\circ\varphi(x)}{{\rm d}x}}{\rm d}x =
\int_{\RR} \abr{\frac{{\rm d}\omega\circ\varphi(x)}{{\rm d}\varphi(x)}\frac{{\rm d}\varphi(x)}{{\rm d}x}}{\rm d}x =
\int_{\RR} \abr{\frac{{\rm d}\omega\circ\varphi(x)}{{\rm d}\varphi(x)}}\frac{{\rm d}\varphi(x)}{{\rm d}x}{\rm d}x =
\int_{\RR}  \abr{\omega'(\varphi)}{\rm d}\varphi.
\end{align*}
\begin{property}
Given a strictly monotonic function $\varphi_j:\RR\to\RR$ 
for each $j\in\range1p$, consider modified samples
\begin{align}
(\hat{\ex}_i,y_i)_{i\in\range1m} \triangleq (\varphib\rbr{\ex_i},y_i)_{i\in\range1m} ,
\end{align}
where $\varphib\rbr{\ex} = (\varphi_j(x_j))_{j\in\range1p}$.
Then for any minimizer $f^*$ in \eqref{eq:FLR0} with respect to $(\hat{\ex}_i,y_i)_{i\in\range1m}$,  
$f^*\circ\varphib$ achieves the minimum of \eqref{eq:FLR0} with respect to $(\ex_i,y_i)_{i\in\range1m}$.
\end{property}
\begin{proof}
For any $f$ in \eqref{eq:FLR0} with respect to $(\hat{\ex}_i,y_i)_{i\in\range1m}$,  
$f\circ\varphib$ yields the same value of objective function in \eqref{eq:FLR0} with respect to $(\ex_i,y_i)_{i\in\range1m}$. It can be seen from $f\circ\varphib = \sum_{j\in\range1p} f_j\circ\varphi_j$ and 
\begin{align*}
\sum_{i\in\range1m} \ell \rbr{f(\hat{\ex}_i), y_i } + \lambda \sum_{j\in\range1p} \nnbr{f_j(\cdot)}
=\sum_{i\in\range1m} \ell \rbr{f\circ\varphib(\ex_{i}), y_i } + \lambda \sum_{j\in\range1p} \nnbr{f_j\circ\varphi_j(\cdot)}.
\end{align*}
\end{proof}
This property indicates that training based on TV-regularization is invariant under transformations of 
the explanatory variables such as scaling, shifting and even nonlinear transformations by monotonic functions. In the sense that we can obtain an optimal predictor among such  transformation without any prior knowledge, this property is very important from a practical point of view.

Third, we see that training based on TV-regularization is reduced to a minimization problem with a finite number of parameters.
\begin{property}
For a minimization problem defined in \eqref{eq:FLR0}, consider the following minimization problem:
\begin{align}
\label{eq:FLR_L1}
%J(\vw) \triangleq 
{ \rm minimize}_{w_{j,s,t}}
\sum_{i\in\range1m}\ell\rbr{\sum_{(j,s,t)\in J} 2w_{j,s,t} \phi_{j,s,t}(\ex_i),y_i}+ 2\lambda\sum_{(j,s,t)\in J}\abr{w_{j,s,t}},
\end{align}
where $J= \cbr{(j,s,t)\in\range 1p\times \range1m\times\range1m\middle| s \le t}$ and
\begin{align}
\phi_{j,s,t}(\ex) \triangleq \frac 12[\![ x_{i(s)j}\le x_{j} < x_{i(t+1)j}  ]\!].
\end{align}
Here, $i(t)$ denotes the index of
a sample at which the $j$-th explanatory 
variable is the $t$-th smallest among $(x_{ij})_{i\in[1,m]}$.
Let $(w^*_{j,s,t})_{(j,s,t)\in J}$ be any minimizer of \eqref{eq:FLR_L1}.
Then, the minimum value of \eqref{eq:FLR0}
is achieved by $f^*$ where
\begin{align}
f^*(\ex) = \sum_{(j,s,t)\in J} 2w^*_{j,s,t} \phi_{j,s,t}(\ex).
\end{align}
\end{property}
\begin{proof}
In \eqref{eq:FLR0}, the first term of 
the objective function only depends on function values at observed samples.
Meanwhile, while conditioning $v_{ij} = f_j(x_{ij})$ for $i\in\range1m$, the problem of finding 
$f_j(\cdot)$ that minimizes TV and its minimum 
value can be analytically solved.
The minimum is, for instance, achieved by the following function:
\begin{align}\label{eq:f_jbyv_jt}
f_j(\cdot)  = \sum_{t\in\range1{m+1}} 2v_{i(t)j} \phi_{j,t-1,t}(\cdot).
\end{align}
%\end{align*}
Note that we defined exceptionally that $x_{i(0)j} \triangleq -M$ and $x_{i(m+1)j} \triangleq M$ for
sufficiently large $M$ and that $i(\cdot)$ implicitly depends on $j$.
The optimality follows directly from the definition.
Its total variation is expressed as
\begin{align}\label{eq:fusedlasso}
|v_{i(0)j}|+\!\!\sum_{t\in\range1{m-1}}|v_{i(t)j}-v_{i(t+1)j}|+|v_{i(m)j}| .
\end{align}
Moreover, we can see that there exists $(w_{j,s,t})_{j\in J}$
which satisfies
\begin{align}\label{eq:fusedlasso}
&\sum_{(j,s,t)\in J} 2w_{j,s,t} \phi_{j,s,t}(\cdot) = \sum_{j\in\range1p}\sum_{t\in\range1m} 2v_{i(t),j} \phi_{j,t-1,t}(\cdot)
\quad {\rm and}\\
&2\sum_{(j,s,t)\in J} |w_{j,s,t}| = \sum_{j\in\range1p} \rbr{|v_{i(0)j}|+\!\!\sum_{t\in[1,m-1]}|v_{i(t)j}-v_{i(t+1)j}|+|v_{i(n)j}| }.
\end{align}
This technical lemma is proved in Lemma~\ref{lemma:main}.
Substituting these equations, we obtain \eqref{eq:FLR_L1}. 
\end{proof}

From this property, 
we can solve \eqref{eq:FLR_L1} to find 
the solution of \eqref{eq:FLR0}.
 Therefore, when $\ell(\cdot,y)$ is convex for any $y\in Y$,
it is boiled down to a convex minimization problem.
Furthermore, as the second term has a separable structure, the coordinate-wise stationary condition guarantees the global solution when $\ell(\cdot,y)$ is also smooth.
In this case, not only \eqref{eq:FLR0} is
boiled down to a minimization problem with a finite number of parameters, but it can also be solved computationally efficiently.
%That is, we can see the following condition for $(w_{jst})$
%guarantees the optimality of the solution:
%\begin{align}
%\forall j,s,t,\quad \partial J(w_{jst}) \ni 0,
%%\Bigg[\!\!\Bigg[\abs{ \frac{\partial L}{\partial w_{jst}} } \le 1 \Bigg]\!\!\Bigg][\![w_{jst}=0]\!] +
%\end{align}
%where $\partial J(w_{jst})$ denotes the subdifferential
%f the objective function with respect to $w_{jst}$.

Lastly, we state a property used in the proof of Lemma~\ref{lemma:main}.
\begin{property}
  For any $x< x'$, let $\phi_{x,x'}(\cdot) = \half [\![ x \le \cdot < x' ]\!]$. Then $\nnbr{\phi_{x,x'}} = 1$.
\end{property}

\section{Technical Lemma}
\newcommand{\prodle}{\circ}
In this section, before the main theorem (Theorem~\ref{thm:main}) regarding the empirical Rademacher complexity of ${\rm GAM}_p(C)$, we prove a lemma (Lemma~\ref{lemma:ineq}) used in its proof. Here, we define $A\prodle B \triangleq\cbr{(a,b)\in A\times B|a\le b}$ for $A,B\subset \NN$.
\begin{lemma} \label{lemma:ineq}
  Let $\gamma_1,\gamma_2,\ldots,\gamma_m$ and  $x_1 \le x_2 \le \ldots \le x_m$ be all real numbers.
  Then we have:
  \begin{align}
  2\sup_{f\in {\rm GAM}_1(1)} \sum \gamma_i f(x_i) \le \max_{i\in\range 0m}\Gamma_i  - \min_{i\in\range 0m} \Gamma_i,
  \end{align}
  where $\Gamma_i \triangleq \sum_{j\in \range1i} \gamma_j$.
\end{lemma}
\begin{proof}
Using Lemma~\ref{lemma:main}, the value of $\sup_{f\in {\rm GAM}_1(1)} \sum \gamma_i f(x_i)$ can be expressed as follows:
\begin{align}
  \sup_{f\in {\rm GAM}_1(1)} \sum \gamma_i f(x_i)  = \sup\cbr{\sum_{(i,j)\in \range1m\prodle\range1m} \Gamma_{ij}  w_{ij} : \sum_{(i,j)\in \range1m\prodle\range1m} |w_{ij}| \le \half },
\end{align}
where $\Gamma_{ij} \triangleq \sum_{ i' \in \range ij} \gamma_{i'}$.
Obviously, this value is obtained by $\max_{(i,j)\in \range1m\prodle\range1m} \half |\Gamma_{ij}|$.
Therefore, as we can easily see
$\half \max_{(i,j)\in\range1m\prodle\range1m} |\Gamma_{ij}| \le \half \rbr{\max_{i\in\range 0m}\Gamma_i  - \min_{i\in\range 0m} \Gamma_i }$,
the claim of the lemma holds true.
\end{proof}
In what follows, we now state and prove lemma~\ref{lemma:main} used in the above.
\begin{lemma}\label{lemma:main}
  Let $\gamma_1,\gamma_2,\ldots,\gamma_m$ and  $x_1 \le x_2 \le \ldots \le x_m$ be all real numbers.
  Then, the following statements are all equivalent:
  \begin{enumerate}[leftmargin=17pt]
    \item there exists $f\in {\rm GAM}_1(1)$ such that $\sum_{i\in\range1m} \gamma_i f(x_i) = \alpha$.
    \item there exists $(v_i)_{i \in \range1m}$ such that 
    \begin{align}\label{eq:v_tv}
    |v_1| + \sum_{i\in \range1{m-1}} {|v_i -v_{i+1}|} + |v_m| \le 1
    \end{align} and $\sum_{i\in \range1m} \gamma_i v_i = \alpha$.
    \item there exists $(w_{ij})_{(i,j)\in\range1m\prodle\range1m}$ such that
      \begin{align}
        &2\sum_{(i,j)\in \range1m\prodle\range1m} |w_{ij} | \le 1\label{eq:w_tv}, \quad {\rm and}\\
        &\sum_{(i,j)\in \range1m\prodle\range1m} \Gamma_{ij} w_{ij} = \alpha \label{eq:w_gs} ,
      \end{align}
      where $\Gamma_{ij} = \sum_{ i' \in \range ij} \gamma_{i'}$.
  \end{enumerate}  
\end{lemma}
\begin{proof}
  $1\to 2:$\\
  For $f\in {\rm GAM}_1(1)$ such that $\sum_{i\in \range1m} \gamma_i f(x_i) = \alpha$, set $v_i = f(x_i)$ for $i\in \range1m$.
  From the definition of total variation, we can see
  \begin{align}
     |f(x_1)| + \sum_{i\in\range1{m-1}} {|f(x_i) -f(x_{i+1})|} + |f(x_m)| \le \nnbr{f} \le 1.
  \end{align}
  Then, $|v_1| + \sum_{i\in\range1{m-1}} {|v_i -v_{i+1}|} + |v_m| \le 1$ and $\sum_{i\in\range1m} \gamma_i v_i = \alpha$ are satisfied with $(v_i)_{i\in\range1m}$.
  
  $2\to 3:$\\  
  We first prove that for any $(v_i)_{i\in \range1m}$, there exists $(w_{ij})_{(i,j)\in\range1m\prodle\range1m}$ such that
  \begin{align}
    & 2\sum_{(i,j)\in\range1m\prodle\range1m} |w_{ij}|= |v_1| + \sum_{i\in \range1{m-1}} {|v_i -v_{i+1}|} + |v_m|,  
    \quad {\rm and } \label{eq:lemma_tv}\\
    & \sum_{i\in \range1{i'}} \sum_{j\in \range{i'}m}  w_{ij} = v_{i'},  \quad  \forall i'\in\range1m ,\label{eq:lemma_sum}
  \end{align}
  by induction with respect to $m$.
%  Then, such $(w_{ij})_{(i,j)\in (\range1m)^2 :i\le j}$
%  suffices 3.\ in the statement under the assumption of 2.
  When $m=1$, setting $w_{11}=v_1$ immediately gives \eqref{eq:lemma_tv} and \eqref{eq:lemma_sum}.
  For $m\ge 2$, let $i^\star$ be the smallest index such that
  $v_{i^\star} =\max_{i} v_i $.
  For the simplicity of the notation, we set $v_0 = v_{m+1} =0$. Then 
  \eqref{eq:lemma_tv} can be written as
  \begin{align}
    2\sum_{(i,j)\in\range1m\prodle\range1m} |w_{ij}|= \sum_{i\in \range 0m} {|v_i -v_{i+1}|} .
  \end{align}
  By the inductive assumption, for $(v_i)_{i\in\range1m\setminus
    i^\star}$, there exists $(w'_{ij})_{(i,j)\in (\range1m\setminus i^\star) \prodle(\range1m\setminus i^\star)}$ such that
  \begin{align}
    & 2\sum_{(i,j)\in \range1m\prodle\range1m} |w'_{ij}|= |v_{i^\star +1} - v_{i^\star -1}| + \sum_{i\in \range 0{m-1}\setminus  \range{i^\star-1}{i^\star} } {|v_i -v_{i+1}|} ,
    \quad {\rm and }\\
    & \sum_{i\in \range1{i'}\setminus i^\star} \sum_{j\in \range{i'}m\setminus i^\star}  w'_{ij} = v_{i'}  \quad  i'\in\range1m\setminus i^\star.
  \end{align}
  There are two possible cases: $v_{i^\star +1} < v_{i^\star -1}$ and
  $v_{i^\star +1} \ge v_{i^\star -1}$.  In the case where $v_{i^\star +1} <
  v_{i^\star -1}$, we prove that $(w_{ij})_{(i,j)\in \range1m\prodle\range1m}$ defined as follows satisfies \eqref{eq:lemma_tv} and
  \eqref{eq:lemma_sum}:
  \begin{align}
    w_{ij} = \begin{cases}
      w'_{ij} & (i,j) \in (\range1m\setminus i^\star) \prodle (\range1m\setminus \range{i^\star-1}{i^\star}) \\
      v_{i^\star}-v_{i^\star-1} & i=i^\star, j=i^\star \\
      0  & i=i^\star , j\in \range{i^\star+1}m \\
      0  & i \in \range1{i^\star-1}, j =i^\star-1 \\
      w'_{ii^\star-1} & i \in \range 1{i^\star-1}, j=i^\star.
    \end{cases}
  \end{align}
  We can see \eqref{eq:lemma_tv} as follows:
  \begin{align}
    2\sum_{(i,j)\in\range1m\prodle\range1m} |w_{ij}|
    &= 2\sum_{(i,j)\in (\range1m\setminus i^\star) \prodle(\range1m\setminus i^\star)} |w'_{ij}| + 2|v_{i^\star}-v_{i^\star-1}| \\
    &=  |v_{i^\star +1} -  v_{i^\star -1}| + \sum_{i\in \range 0m\setminus  \range{i^\star-1}{i^\star} } {|v_i -v_{i+1}|
    + 2|v_{i^\star}-v_{i^\star-1}|} \\
    &=\sum_{i\in \range 0m\setminus \range{ i^\star-1}{i^\star} } {|v_i -v_{i+1}|}
     +v_{i^\star -1} -  v_{i^\star +1}  + 2v_{i^\star}-2v_{i^\star-1} \\
    &= \sum_{i\in \range 0m} {|v_i -v_{i+1}|}.
  \end{align}
  As for \eqref{eq:lemma_sum}, when $i' \in \range1{i^\star -1}$,
  \begin{align}
    \sum_{i\in \range1{i'}} \sum_{j\in \range{i'}m}  w_{ij}
    &=  
    \sum_{i\in \range1{i'}} \rbr{\sum_{j\in \range{i'}m\setminus \range{i^\star-1}{i^\star}}  w_{ij} + w_{ii^\star-1} + w_{ii^\star} } \\
    &=  
    \sum_{i\in \range1{i'}} \rbr{\sum_{j\in \range{i'}m\setminus \range{i^\star-1}{i^\star}}  w'_{ij} + 0 +  w'_{ii^\star-1}  } \\
    &=  
    \sum_{i\in \range1{i'}\setminus i^\star} \rbr{\sum_{j\in \range{i'}m\setminus i^\star}  w'_{ij} } \\
    &= v_{i'} .
  \end{align}
  When $i' = i^\star$,
  \begin{align*}
    \sum_{i\in \range1{i^\star}} \sum_{j\in \range{i^\star}m}  w_{ij}
    &=  
    \sum_{i\in \range1{i^\star}} \rbr{\sum_{j\in \range{i^\star+1}m}  w_{ij} + w_{i i^\star} } \\
    &=  
    \sum_{i\in \range1{i^\star-1}} \rbr{ \sum_{j\in \range{i^\star+1}m}  w_{ij}
                                  +
                                  w_{i i^\star}}
    +
    \sum_{j\in \range{i^\star+1}m}  w_{i^\star j} + w_{i^\star i^\star}
    \\&=
    \sum_{i\in \range1{i^\star-1}} \rbr{ \sum_{j\in \range{i^\star+1}m }  w'_{ij}
                                  +
                                  w'_{i i^\star-1}}
    + 0 + v_{i^\star} - v_{i^\star-1}
    \\&=
    \sum_{i\in \range1{i^\star-1} \setminus i^\star} \sum_{j\in \range{i^\star-1}m \setminus i^\star}  w'_{ij}
    + v_{i^\star} - v_{i^\star-1}
    \\&= v_{i^\star-1} + v_{i^\star} - v_{i^\star-1} =  v_{i^\star}.
  \end{align*}
  When $i' \in \range{i^\star +1}m$,
  \begin{align}
    \sum_{i\in \range1{i'}} \sum_{j\in \range{i'}m}  w_{ij}
    &=
    \sum_{i\in \range1{i'} \setminus i^\star } \sum_{j\in \range{i'}m}  w_{ij} + \sum_{j\in \range{i'}m}  w_{i^\star j}
    \\&=  
    \sum_{i\in \range1{i'}\setminus i^\star} \sum_{j\in \range{i'}m\setminus \range{i^\star-1}{i^\star}}  w_{ij}
    \\&=  
    \sum_{i\in \range1{i'}\setminus i^\star} \sum_{j\in \range{i'}m\setminus \range{i^\star-1}{i^\star}}  w'_{ij}
    = v_{i'}.
  \end{align}
  Therefore, \eqref{eq:lemma_sum} holds.
  In the case where $v_{i^\star +1} \ge v_{i^\star -1}$, set $(w_{ij})_{(i,j)\in\range1m\prodle\range1m}$ as follows:
  \begin{align}
    w_{ij} = \begin{cases}
      w'_{ij} & (i,j) \in (\range1m\setminus \range{i^\star}{i^\star+1}) \prodle (\range1m\setminus i^\star) \\
      v_{i^\star}-v_{i^\star+1} & i=i^\star, j=i^\star \\
      w'_{i^\star+1j} & i = i^\star , j \in \range{i^\star+1}m, \\
      0  & i = i^\star+1 , j \in \range{i^\star+1}m, \\
      0 & i\in \range1 {i^\star-1}, j=i^\star .
    \end{cases}
  \end{align}
  As for \eqref{eq:lemma_tv}, a similar argument as above holds as follows:
  \begin{align}
    2\sum_{(i,j)\in \range1m\prodle\range1m} |w_{ij}|
    &=  |v_{i^\star +1} -  v_{i^\star -1}| + \sum_{i\in \range 0m\setminus  \range{i^\star-1}{i^\star} } {|v_i -v_{i+1}|
    + 2(v_{i^\star}-v_{i^\star+1})} \\
    &= \sum_{i\in \range 0m} {|v_i -v_{i+1}|}.
  \end{align}  
  Also for \eqref{eq:lemma_sum}, when $i' \in \range1{i^\star -1}$,
  \begin{align}
    \sum_{i\in \range1{i'}} \sum_{j\in\range{i'}m}  w_{ij}
    =
    \sum_{i\in \range1{i'} } \sum_{j\in\range{i'}m \setminus i^\star }  w_{ij} 
    =  
    \sum_{i\in \range1{i'}\setminus \range{i^\star}{i^\star+1}} \sum_{j\in \range{i'}m\setminus i^\star}  w_{ij}
    = v_{i'}.
  \end{align}
  When $i'=i^\star$,
  \begin{align*}
    \sum_{i\in \range1{i^\star}} \sum_{j\in\range{i^\star}m}  w_{ij}
    &=  
    \sum_{i\in \range1{i^\star-1}} \rbr{ \sum_{j\in \range{i^\star+1}m}  w_{ij}
                                   +
                                   w_{i i^\star}}
    +
    \sum_{j\in \range{i^\star+1}m}  w_{i^\star j}
    + w_{i^\star i^\star}
    \\&=
    \sum_{i\in \range1{i^\star-1}} \rbr{ \sum_{j\in \range{i^\star+1}m }  w'_{ij}
    +
    0 }
    + \sum_{j\in i^\star+\range1m}  w'_{i^\star+1 j}  + v_{i^\star} - v_{i^\star+1}
    \\&=
    \sum_{i\in \range1{i^\star+1} \setminus i^\star} \sum_{j\in \range{i^\star+1}m \setminus i^\star}  w'_{ij}
    +v_{i^\star} - v_{i^\star+1}
    \\&= v_{i^\star+1} + v_{i^\star} - v_{i^\star+1} =  v_{i^\star}.
  \end{align*}
  Finally, when $i'\in \range{i^\star +1}m$,
  \begin{align}
    \sum_{i\in \range1{i'}} \sum_{j\in \range{i'}m}  w_{ij}
    &=  
    \sum_{j\in \range{i'}m}    
    \rbr{\sum_{i\in \range1{i'}\setminus \range{i^\star}{i^\star+1}}  w_{ij} + w_{i^\star j} + w_{i^\star+1 j} } \\
    &=  
    \sum_{j\in \range{i'}m}
    \rbr{\sum_{i\in \range1{i'}\setminus \range{i^\star}{i^\star+1}}  w'_{ij} + w'_{i^\star+1 j} +0  } \\
    &=  
    \sum_{j\in \range{i'}m\setminus i^\star}
    \rbr{\sum_{i\in \range1{i'}\setminus i^\star}  w'_{ij} } 
    = v_{i'} .
  \end{align}
  Therefore, \eqref{eq:lemma_tv} and \eqref{eq:lemma_sum} hold for any $m\in\NN$.
  We then show such $(w_{ij})_{(i,j)\in\range1m\prodle\range1m}$ satisfies \eqref{eq:w_tv}  and \eqref{eq:w_gs} given that 2.\ holds. \eqref{eq:w_tv} can be immediately seen by \eqref{eq:v_tv}
  and \eqref{eq:lemma_tv}. For \eqref{eq:w_gs}, it holds from
   $\sum_{i\in\range1m} \gamma_i v_i =\alpha$ and the following relation:
  \begin{align*}
    \sum_{(i,j)\in \range1m\prodle\range1m} \Gamma_{ij} w_{ij} = 
    \sum_{(i,j)\in\range1m\prodle\range1m} \sum_{i'\in \range ij} \gamma_{i'} w_{ij} = 
    \sum_{i'\in \range1m} \sum_{(i,j)\in \range1{i'}\times\range{i'}m } 
    \gamma_{i'} w_{ij} = \sum_{i'\in \range 1m} \gamma_{i'} v_{i'}. 
  \end{align*}
  
  $3\to 1:$\\
  For such $(w_{ij})_{(i,j)\in\range1m\prodle\range1m}$, set $f(\cdot) $ as $\sum_{(i,j)\in\range1m\prodle\range1m} 2w_{ij} \phi_{i,j} (\cdot)$, where $\phi_{i,j}(\cdot) = \half [\![ x_i \le \cdot \le x_j ]\!]$.
  As $\nnbr{\phi_{i,j}} = 1$,
  \begin{align}
    \nnbr{f} \le \sum_{(i,j)\in\range1m\prodle\range1m} 2 |w_{ij}| \nnbr{\phi_{i,j} (\cdot)} \le \sum_{(i,j)\in\range1m\prodle\range1m}  2|w_{ij}| \le 1,
  \end{align}
  which means $f \in {\rm GAM}_1(1)$.
  On the other hand, because $f(x_{i'})= \sum_{(i,j)\in \range1{i'} \times \range{i'}m} w_{ij}$, we see
  \begin{align}
    \sum_{i'\in\range 1m} \gamma_{i'} f(x_{i'}) =  
    \sum_{i'\in\range 1m} \gamma_{i'} \sum _{(i,j) \in \range1{i'}\times \range{i'}m } w_{ij}
   % =\sum_{(i,i',j) \in \range1m\prodle\range1m\prodle\range1m } \gamma_{i'}w_{ij}
    =\sum_{(i,j) \in\range1m\prodle\range1m}  \Gamma_{ij}w_{ij}
    =\alpha.
  \end{align}
\end{proof}
\section{Main Result}
In this section, we state the main theorem on the empirical Rademacher complexity and the corollary on generalization bounds.
\begin{theorem}\label{thm:main} 
  Let ${\rm GAM}_{p,\ell}(C) \triangleq \cbr{ (\ex,y) \mapsto
    \ell(f(\ex),y) \middle| f\in {\rm GAM}_p(C) } \subset L^1(\RR^p\times Y)$
    for a loss function $\ell:\RR \times Y \to \RR_+$
    in which $\ell(\cdot,y)$ is $\rho$-Lipschitz
for any $y\in Y$.
Then, for any $(\ex_i,y_i)_{i\in \range 1m}$ and $d>2$, it holds that
\begin{align}\label{eq:g}
  G({\rm GAM}_{p,\ell}(C), (\ex_i,y_i)_{i\in \range1m}) \le \sqrt{\frac 2\pi}\rho C\sqrt{ \frac{5\ceil{\log p}}{m}},
\end{align}
and
\begin{align}\label{eq:r}
  R({\rm GAM}_{p,\ell}(C), (\ex_i,y_i)_{i\in \range1m}) \le \rho C \sqrt{ \frac{5\ceil{\log p}}{m}}.
\end{align}
\end{theorem}
\begin{proof}
  We can easily see that \eqref{eq:g} implies \eqref{eq:r} from the following inequality:
  \begin{align}
    \EE_{\gammab} \sup_{f\in F} \sum \gamma_i f(x_i) &=
    \EE_{\epsilonb} \EE_{\gammab} 
    \sup_{f\in F} \sum \varepsilon_i |\gamma_i| f(x_i) \\
    &\ge 
    \EE_{\epsilonb} \sup_{f\in F} \sum \varepsilon_i \EE_{\gamma_i} |\gamma_i| f(x_i) \\
   & = 
    \EE_{\epsilonb} \EE_{\gamma_1} |\gamma_1|  \sup_{f\in F} \sum \varepsilon_i f(x_i),
  \end{align}
 and $\EE_{\gamma_1} |\gamma_1| =  \sqrt{\frac 2{\pi}}$.
  Therefore, because of the properties of the Rademacher complexity~\citep[Lemma 26.6, 26.9]{SSS14},
  it is sufficient to prove that 
\begin{align}
  G({\rm GAM}_p(1), (\ex_i)_{i\in \range1m}) \le \sqrt{ \frac{5\ceil{\log p}}{m}}\sqrt{\frac 2\pi}.
\end{align}
First, for any $r \ge 1$, we can see that
\begin{align}
  \sup_{f\in {\rm GAM}_p(1)} \sum_{i\in\range1m} \gamma_i f(\ex_i) 
  &= \sup \cbr{ \sum_{j\in\range1p} c_j \sup_{ f\in {\rm GAM}_1(1)} \sum_i \gamma_i f_j(x_{ij}) : c_j \ge 0, \sum_j{c_j} \le 1} \\
  &=  \max_{j\in\range1p} \sup_{f_j\in {\rm GAM}_1(1)} \sum_i \gamma_i f_j(x_{ij}) \\
  &\le \rbr{\sum_{j\in\range1p} \rbr{\sup_{f_j\in {\rm GAM}_1(1)} \sum_i \gamma_i f_j(x_{ij})}^r }^{\frac 1r}.
\end{align}
From Lemma~\ref{lemma:ineq},
  \begin{align}
    2\sup_{f\in {\rm GAM}_1(1)} \sum \gamma_i f(x_i)  \le \max_{i\in\range0m}\Gamma_i  - \min_{i\in\range0m} \Gamma_i,
  \end{align}
  where $\Gamma_i = \sum_{j\in\range1i} \gamma_j$. Therefore, $\sup_{f\in {\rm GAM}_1(1)} \sum \gamma_i f(x_i) > t$
  implies that $\max_{i\in\range0m}\Gamma_i  - \min_{i\in\range0m} \Gamma_i> 2t$, which then implies at least either $\max_{i\in\range0m}\Gamma_i > t$ or $\min_{i\in\range0m}\Gamma_i < -t$ holds. Therefore, for any $t>0$, it holds that
  \begin{align}
    \PP\cbr{\sup_{f\in{\rm GAM}_1(1)} \sum_{i\in\range1m}\gamma_i f(x_i) >t} &\le \PP\cbr{ \max_{i\in\range0m}\Gamma_i > t} +
    \PP\cbr{ \min_{i\in\range0m}\Gamma_i < -t} 
    \\&= 2\PP\cbr{ \max_{i\in\range0m}\Gamma_i > t}.
  \end{align}
 From Levy inequality~\citep{LT91}, we see 
  \begin{align}
    \PP\cbr{ \max_{i\in\range0m}\Gamma_i > t} \le 2\PP\cbr{ \Gamma_m > t} = 2\int_{s=t }^{\infty} \frac{1}{\sqrt{2\pi m}}e^{-\frac{s^2}{2m}} {\rm d}s.
  \end{align}
Therefore, for any $j$, 
\begin{align}
  \EE_{\gammab} \rbr{\sup_{f_j\in {\rm GAM}_1(1)} \sum_{i\in\range1m} \gamma_i f_j(x_{ij})}^r
  &= \int_{t=0}^\infty \PP\cbr{\abr{\sup_{f_j\in {\rm GAM}_1(1)} \sum_{i\in\range1m} \gamma_i f_j(x_{ij})}^r >t } {\rm d}t \\
  &= \int_{t=0}^\infty \PP\cbr{ \sup_{f_j\in {\rm GAM}_1(1)} \sum_{i\in\range1m} \gamma_i f_j(x_{ij}) >t^{\frac 1r} } {\rm d}t  \\
  &\le \int_{t=0}^\infty 2\PP\cbr{ \max_{i\in\range 0m}\Gamma_i > t^{\frac 1r}} {\rm d}t \\
  &\le \int_{t=0}^{\infty}
    4\int_{s=t^{\frac 1r} } ^{\infty} \frac{1}{\sqrt{2\pi m}}e^{-\frac{s^2}{2m}} {\rm d}s{\rm d}t\\
    &=4
    \int_{s=0}^{\infty}
    \int_{t=0} ^{s^r} \frac{1}{\sqrt{2\pi m}}e^{-\frac{s^2}{2m}} {\rm d}s{\rm d}t \\
    &=
    4\int_{s=0}^{\infty}
    \frac{s^r}{\sqrt{2\pi m}}e^{-\frac{s^2}{2m}} {\rm d}s \\
    &= 2\EE_{s\sim {\rm Normal}(0,m)} [ |s|^r] \\
    &= 2(2m)^{\frac r2} \frac{\Gamma\rbr{\frac{r+1}2}}{\sqrt{\pi}}.
\end{align}
Finally, we see
\begin{align}
 \EE_{\gammab}  \sup_{f\in {\rm GAM}_p(1)} \sum_{i\in\range1m} \gamma_i f(\ex_i)
&\le  \EE_{\gammab} \rbr{\sum_{j\in\range1p} \rbr{\sup_{f_j\in {\rm GAM}_1(1)} \sum_{i\in\range1m} \gamma_i f_j(x_{ij})}^r}^{\frac 1r}\\
&\le \rbr{ \EE_{\gammab} \sum_{j\in\range1p} \rbr{\sup_{f_j\in {\rm GAM}_1(1)} \sum_{i\in\range1m} \gamma_i f_j(x_{ij})}^r}^{\frac 1r}\\
&\le \rbr{p\cdot 2(2m)^{\frac r2} \frac{\Gamma\rbr{\frac{r+1}2}}{\sqrt{\pi}}}^{\frac 1r} \label{aaa}\\
& < \sqrt{2m} \rbr{ \frac{2p}{\sqrt{\pi}} \sqrt{2\pi} \rbr{\frac{s}{2e}}^{\frac s2}e^{\frac 1{6s}} }^{\frac 1{1+s}}.
\end{align}
We set $r = 1+s $ and used $\Gamma(1+\frac s2) < \sqrt{2\pi} \rbr{\frac s{2e}}^{\frac s2}e^{\frac 1{6s}}$ in  the last inequality.
Setting $s=2\ceil{\log p}$, 
\begin{align}
 \sqrt{2m} \rbr{ 2\sqrt{2} e^{\frac s2} \rbr{\frac{s}{2e}}^{\frac s2}e^{\frac 1{6s}}}^{\frac 1{1+s}}
&
 = \sqrt{2m} \sqrt {\frac{s}{2} } \rbr{ 2\sqrt{2} \sqrt{\frac{2}{s}}  e^{\frac 1{6s}} }^{\frac 1{1+s}} \\&
 = \sqrt{2m \ceil{\log p}}  \rbr{ \frac 4{ \sqrt{s}}  e^{\frac 1{6s}} }^{\frac 1{1+s}} \\&
 < \sqrt{2m \ceil{\log p}}\sqrt{\frac 5\pi}. \label{eq:last}
\end{align}
\eqref{eq:last} holds when $ s \ge 4$ because $\rbr{ \frac 4{ \sqrt{s}}  e^{\frac 1{6s}} }^{\frac 1{1+s}}$ is maximized
at $s=4$ in the range $ s \ge 4$ and $\rbr{ 2  e^{\frac 1{24}} }^{\frac 15}$ is less than $\sqrt{\frac{5}{\pi}}$
\footnote{ in case of $p=2$, setting $r=3$ in \eqref{aaa} yields a similar bound with $\sqrt{\frac 6 \pi }$ instead of $\sqrt{\frac 5 \pi }$}.
\end{proof}
Lastly, we state the generalization bound that can be derived directly from the result in \citep[Theorem 26.5]{SSS14}.
\begin{corollary}\label{coro}
Assume that $(\ex,y)$ and $(\ex_i,y_i)_{i\in\range1m}$ are
i.i.d.\ random variables on $\RR^p\times Y$ and $\ell(\cdot,y)$ is
$\rho$-Lipschitz and bounded by $c>0$ for any $y\in Y$. Then, the following statements hold true for $p>2$ and $\delta>0$:
\begin{enumerate}[leftmargin=17pt]
\setlength\itemsep{-9pt}
\item For $\epsilon = \max_{f \in {\rm GAM}_p(C)} \rbr{ \EE_{\ex,y} \ell(f(\ex),y) - \frac 1m\sum_{i\in\range1m} \ell(f(\ex_i),y_i))}$, 
  \begin{align*}
    \PP_{(\ex_i,y_i)_{i\in\range1m}}\cbr{  \epsilon \le  \rho C \sqrt{ \frac{5\ceil{\log p}}{m}} +c\sqrt{\frac{2\log(2/\delta)}{m}} } \ge 1-\delta.
  \end{align*}
\item 
  For $\ell^* = \inf_{f\in{\rm GAM}_p(C) } \EE_{\ex,y}\ell(f(\ex),y)$ and $\fhat = \argmin_{f \in {\rm GAM}_p(C)}  \frac 1m\sum_{i\in\range1m} \ell(f(\ex_i),y_i)$, 
  \begin{align*}
        \PP_{(\ex_i,y_i)_{i\in\range1m}}\cbr{ \EE_{\ex,y}  \ell(\fhat(\ex),y) \le \ell^* + \rho C \sqrt{ \frac{5\ceil{\log p}}{m}} + 5c\sqrt{\frac{2\log(2/\delta)}{m}} } \ge 1-\delta.
  \end{align*}
\end{enumerate}
\end{corollary}
\section{Discussion}
\subsection{Related Work}
In the literature that directly deals with GAM, the most of theoretical results are based on the assumption that the true distribution is contained in the model~\citep{Wood06}. Although GAM can be seen as classical nonparametric classification or regression when $p=1$, in which there is substantial work on the distribution-free theory~\citep{Gyo06}, there is no work related to the distribution-free theory in the context of GAM. To the best of our knowledge, one of the closest result to ours
is work by \citet{Cor10}, in which the authors
studied the Rademacher complexity of the following hypothesis class in the context of multiple kernel learning:
\begin{align*}
H^1_p =\cbr{  f(\ex) =\sum_{j\in\range1p} \mu_j \omega_j(\ex) \middle| \sum_{j\in\range1p}\mu_j\nbr{\omega_j}^2_{\Hcal_j}\le 1 , \sum_{j\in\range1p} \mu_j \le 1, \mu_j \ge 0},
\end{align*}
where $\nbr{\cdot}_{\Hcal_j}$ denotes the norm in RKHS $\Hcal_j$.
%Here, $\Phi_j(\ex) \triangleq K_j(\ex,\cdot)$ in which $K_j$ is the reproducing kernel of RKHS $\Hcal_j$ and  $\nbr{\cdot}_{\Hcal_j}$ and $\inner{\cdot}{\cdot}_j$ denote its norm and inner product. 
The authors have shown that $R(H^1_p,(\ex_i)_{i\in\range1m})$ is an order of $O(\sqrt{\nicefrac{\log p}{m}})$.
When we restrict $\Hcal_j$ so all $\omega_j \in \Hcal_j$
to be dependent on the $j$-th explanatory variable only, $H^1_p$
becomes also a class of GAM predictors, where the sum of the norms of weight functions is upper bounded by 1. Note that TV is seen as a norm in $L^1$-space, which is not an RKHS.
%However, each weight function has to be represented as $\mu_j \inner{\omega}{\Phi_j(x_j)}_j$ by the same $\omega$ among all $j$ and it is rather restrictive from the viewpoint of GAM predictors. 

In addition, when we restrict $\nbr{\ex_i}_{\infty} \le 1$, \citet{Kakade09} proved that $R(F_W,(\ex_i)_{\in\range1m})$ is an order of $O(\sqrt{\nicefrac{\log p}{m}})$, 
where $F_W = \cbr{\ex\mapsto \sum_{j\in\range1p} w_jx_j\middle| (w_j)_{j\in\range1p} \in W}$ and $W = \cbr{\bm{w} \in \RR^p | \nbr{\bm{w}}_1\le 1}$.
As it is easy to see $F_W \subset {\rm GAM}_p(2)$, we can view our result as an extension of their result to nonlinear GAM predictors.

\subsection{Tightness}
We consider the result of Theorem~\ref{thm:main}
 in the context of the classification problem, in which $Y=\cbr{\pm1}$ and $X=\cbr{\pm1}^p\subset\RR^p$.
 Then, $J_p=\cbr{\ex\mapsto \pm\sgn(x_j)|j\in\range1p} \subset {\rm GAM}_p(2)$  implies
$R(J_p, (\ex_i)_{i\in\range1m}) \le R({\rm GAM}_{p}(2), (\ex_i)_{i\in\range1m})$. Therefore,
theorem 26.5 in~
\citep{SSS14} implies
\[
\PP_{(\ex_i,y_i)_{i\in\range1m}}\rbr{ \EE_{\ex,y}\ell(\hat{f}(\ex),y)  - \ell^* > 
R(J_p, (\ex)_{i\in\range1m})　 + 5\sqrt{\frac{2\log(2/\delta)}{m} }}  < \delta. 
\]
even for $J_p$, in which $\ell(a,y)=\max\cbr{0,1-ay}$.
As $f(\ex)y \in\cbr{\pm1}$, it holds that $f(\ex)y >0 \Leftrightarrow f(\ex)y  =1$ and $f(\ex)y \le 0 \Leftrightarrow f(\ex)y  = -1$, which implies $[\![f(\ex)y >0]\!] = \ell(f(\ex),y)$ for any $f\in J_p$, $x\in X$, and $y\in Y$. Therefore, 
\begin{align}\label{generalizationerror}
\EE_{(\ex_i,y_i)_{i\in\range1m}}\rbr{\EE_{\ex,y}  [\![ \hat{f}(\ex)y >0 ]\!]  - \inf_{f\in J_p}\EE_{\ex,y}[\![ f(\ex)y >0 ]\!] }
\end{align}
is also an order of $ O\rbr{R({\rm GAM}_{p}(2), (\ex_i)_{i\in\range1m})}$ for any distribution of $(\ex,y)\in X\times Y$.

On the other hand, it is known that,
for any $\tilde{f}(\cdot) $ learned from $m$ i.i.d. samples, there exists  $\PP_{\ex,y}$ such that~\eqref{generalizationerror} is an order of $ \Omega\rbr{\sqrt{\nicefrac{{\rm VCdim}(F)}{m}}}$
under the assumption that $\inf_{f\in F}\EE_{\ex,y}[\![ f(\ex)y >0 ]\!] \ne0$~\citep{DevLug95,BouBouLug05}. 
Because $J_p$ contains $2p$ different classifiers, ${\rm VCdim}(J_p)$ is at least of $\Omega(\log p)$, which implies that \eqref{generalizationerror} is an order of 
$\Omega(\sqrt{\nicefrac{\log p}{m}})$. 

Therefore, $R({\rm GAM}_{p}(1), (\ex_i)_{i\in\range1m})$ cannot be tighter than $O(\sqrt{\nicefrac{\log p}{m}})$.
\ifx10
This is where the content of your paper goes.  Remember to:
\begin{itemize}
\item Limit the main text (without references and appendices) to 12 PMLR-formatted pages (i.e., using this template).
\item Include, either in the main text or the appendices, all details, proofs
  and derivations required to substantiate the results.
\item Include {\em in the main text} enough details, including proof
  details, to convince the reviewers of the contribution, novelty and significance of the submissions.
\item Not include author names (this is done automatically), and to
  the extent possible, avoid directly identifying the authors.  You
  should still include all relevant references, including your own,
  and any other relevant discussion, even if this might allow a
  reviewer to infer the author identities.
 \end{itemize}
\fi

% Acknowledgments---Will not appear in anonymized version
%\acks{We thank a bunch of people.}

\clearpage
\bibliographystyle{plain}
\bibliography{bib2}

%\appendix
%\section{My Proof of Theorem 1}
%This is a boring technical proof.
%\section{My Proof of Theorem 2}
%This is a complete version of a proof sketched in the main text.
\end{document}